\documentclass[conference]{IEEEtran}
\usepackage{times}
\usepackage{algorithm}
\usepackage{algorithmic}
\usepackage{float}
\usepackage{graphicx}
\usepackage{amssymb}
\usepackage{amsmath,amsthm, url}
\usepackage[font=small,labelfont=bf,tableposition=top]{caption}
\usepackage[font=footnotesize]{subcaption}
\usepackage[flushleft,neveradjust]{paralist}
\usepackage{url}
\usepackage{mathtools}
\DeclareMathOperator*{\argmin}{argmin}

\usepackage{array}
\newcolumntype{H}{>{\setbox0=\hbox\bgroup}c<{\egroup}@{}}

\usepackage[numbers]{natbib}
\usepackage{multicol}
\usepackage[bookmarks=true]{hyperref}

\begin{document}

\title{Learning Contracting Vector Fields\\ For Stable Imitation Learning}


\author{\authorblockN{Vikas Sindhwani}
\authorblockA{Google Brain\\
New York City, NY 10011\\
sindhwani@google.com}
\and
\authorblockN{Stephen Tu}
\authorblockA{University of California Berkeley\\
Berkeley, CA 94720\\
stephent@berkeley.edu}
\and
\authorblockN{Mohi Khansari}
\authorblockA{X\\
Mountain View, CA 94043\\
khansari@x.team}}


%

\maketitle
\newtheorem{thm}{Theorem}[section]
\newtheorem{lem}{Lemma}[section]
\newtheorem{prop}{Proposition}
\newtheorem{defn}{Definition}[section]
\newtheorem{condition}{Condition}

\newcommand{\transpose}{{\scriptscriptstyle T}}
\newcommand{\reals}{\mathbb{R}}
\newcommand{\biggram}{\overrightarrow{\mathbf{K}}}
\newcommand{\biggramsep}{{\vK \otimes \vL}}
\newcommand{\veckernel}{\overrightarrow{k}}
\newcommand{\vecf}{\overrightarrow{f}}
\newcommand{\chvec}{\ch_{\veckernel}}
\newcommand{\chvecsep}{\ch_{k\vL}}
\newcommand{\chdict}{\ch_{\cd}}
\newcommand{\expect}{\mathbb{E}}
\newcommand{\integers}{\mathbf{Z}}
\newcommand{\naturals}{\mathbf{N}}
\newcommand{\rationals}{\mathbf{Q}}

\newcommand{\ca}{\mathcal{A}}
\newcommand{\cb}{\mathcal{B}}
\newcommand{\cc}{\mathcal{C}}
\newcommand{\cd}{\mathcal{D}}
\newcommand{\ce}{\mathcal{E}}
\newcommand{\cf}{\mathcal{F}}
\newcommand{\cg}{\mathcal{G}}
\newcommand{\ch}{\mathcal{H}}
\newcommand{\ci}{\mathcal{I}}
\newcommand{\cj}{\mathcal{J}}
\newcommand{\ck}{\mathcal{K}}
\newcommand{\cl}{\mathcal{L}}
\newcommand{\cm}{\mathcal{M}}
\newcommand{\cn}{\mathcal{N}}
\newcommand{\co}{\mathcal{O}}
\newcommand{\cp}{\mathcal{P}}
\newcommand{\cq}{\mathcal{Q}}
\newcommand{\calr}{\mathcal{R}}
\newcommand{\cs}{\mathcal{S}}
\newcommand{\ct}{\mathcal{T}}
\newcommand{\cu}{\mathcal{U}}
\newcommand{\cv}{\mathcal{V}}
\newcommand{\cw}{\mathcal{W}}
\newcommand{\cx}{\mathcal{X}}
\newcommand{\cy}{\mathcal{Y}}
\newcommand{\cz}{\mathcal{Z}}
\newcommand{\pr}{\mathbb{P}}
\newcommand{\predsp}{\cy}  
\newcommand{\outsp}{\cy}

\newcommand{\prxy}{P_{\cx \times \cy}}
\newcommand{\prx}{P_{\cx}}
\newcommand{\prygivenx}{P_{\cy\mid\cx}}
\newcommand{\ex}{\mathbb{E}}
\newcommand{\var}{\textrm{Var}}
\newcommand{\cov}{\textrm{Cov}}
\newcommand{\kl}{\textrm{KL}}
\newcommand{\law}{\mathcal{L}}
\newcommand{\as}{\textrm{ a.s.}}
\newcommand{\io}{\textrm{ i.o.}}
\newcommand{\ev}{\textrm{ ev.}}
\newcommand{\convd}{\stackrel{d}{\to}}
\newcommand{\eqd}{\stackrel{d}{=}}
\newcommand{\del}{\nabla}
\newcommand{\loss}{V}
\newcommand{\risk}{R}
\newcommand{\emprisk}{\hat{R}_{\ell}}
\newcommand{\lossfnl}{\risk}
\newcommand{\emplossfnl}{\emprisk}
\newcommand{\empminimizer}[1]{\hat{#1}_{\ell}} 
\newcommand{\empminimizerfa}{\hat{f}_{\ell}^{1}}
\newcommand{\empminimizerfb}{\hat{f}_{\ell}^{2}}
\newcommand{\minimizer}[1]{#1_{*}} 
\newcommand{\etal}{\textrm{et. al.}}
\newcommand{\rademacher}[1]{\calr_{#1}}
\newcommand{\emprademacher}[1]{\hat{\calr}_{#1}}

\newcommand{\trace}{\operatorname{trace}}
\newcommand{\rank}{\text{rank}}
\newcommand{\linspan}{\text{span}}
\newcommand{\spn}{\text{span}}
\newcommand{\proj}{\text{Proj}}

\newcommand{\bfx}{\mathbf{x}}
\newcommand{\bfy}{\mathbf{y}}
\newcommand{\bfl}{\mathbf{\lambda}}
\newcommand{\bfm}{\mathbf{\mu}}
\newcommand{\calL}{\mathcal{L}}
 
\newcommand{\vX}{\mathbf{X}}
\newcommand{\vY}{\mathbf{Y}}
\newcommand{\vA}{\mathbf{A}}
\newcommand{\vR}{\mathbf{R}}
\newcommand{\vB}{\mathbf{B}}
\newcommand{\vE}{\mathbf{E}}
\newcommand{\vK}{\mathbf{K}}
\newcommand{\vD}{\mathbf{D}}
\newcommand{\vU}{\mathbf{U}}
\newcommand{\vL}{\mathbf{L}}
\newcommand{\vI}{\mathbf{I}}
\newcommand{\vC}{\mathbf{C}}
\newcommand{\vV}{\mathbf{V}}
\newcommand{\vQ}{\mathbf{Q}}
\newcommand{\vM}{\mathbf{M}}
\newcommand{\vT}{\mathbf{T}}
\newcommand{\vS}{\mathbf{S}}
\newcommand{\vN}{\mathbf{N}}
\newcommand{\vZ}{\mathbf{Z}}
\newcommand{\vsig}{\mathbf{\Sigma}}

\newcommand{\vw}{\boldsymbol{w}}
\newcommand{\vx}{\mathbf{x}}
\newcommand{\vxi}{\boldsymbol{\xi}}     
\newcommand{\valpha}{\boldsymbol{\alpha}}
\newcommand{\vbeta}{\boldsymbol{\beta}}
\newcommand{\veta}{\boldsymbol{\eta}}
\newcommand{\vsigma}{\boldsymbol{\sigma}}
\newcommand{\vepsilon}{\boldsymbol{\epsilon}}
\newcommand{\vnu}{\boldsymbol{\nu}}
\newcommand{\vd}{\boldsymbol{d}}
\newcommand{\vs}{\boldsymbol{s}}
\newcommand{\vt}{\boldsymbol{t}}
\newcommand{\vh}{\boldsymbol{h}}
\newcommand{\ve}{\boldsymbol{e}}
\newcommand{\vf}{\boldsymbol{f}}
\newcommand{\vg}{\boldsymbol{g}}
\newcommand{\vz}{\boldsymbol{z}}
\newcommand{\vk}{\boldsymbol{k}}
\newcommand{\va}{\boldsymbol{a}}
\newcommand{\vb}{\boldsymbol{b}}
\newcommand{\vv}{\boldsymbol{v}}
\newcommand{\vr}{\mathbf{r}}
\newcommand{\vy}{\mathbf{y}}
\newcommand{\vu}{\mathbf{u}}
\newcommand{\vp}{\mathbf{p}}
\newcommand{\vq}{\mathbf{q}}
\newcommand{\vc}{\mathbf{c}}
\newcommand{\vW}{\boldsymbol{W}}
\newcommand{\vP}{\boldsymbol{P}}
\newcommand{\vG}{\boldsymbol{G}}
\newcommand{\vH}{\boldsymbol{H}}

\newcommand{\hil}{\ch}
\newcommand{\rkhs}{\hil}
\newcommand{\manifold}{\cm} 
\newcommand{\cloud}{\cc} 
\newcommand{\graph}{\cg}    
\newcommand{\vertices}{\cv} 
\newcommand{\coreg}{{\scriptscriptstyle \cc}}
\newcommand{\intrinsic}{{\scriptscriptstyle \ci}}
\newcommand{\ambient}{{\scriptscriptstyle \ca}}
\newcommand{\isintrinsic}[1]{#1^{\scriptscriptstyle \ci}}
\newcommand{\isambient}[1]{#1^{\scriptscriptstyle \ca}}
\newcommand{\hilamb}{\ch^\ambient}
\newcommand{\hilintr}{\ch^\intrinsic}
\newcommand{\M}{\text{FIX ME NOW}}  
\newcommand{\ktilde}{\tilde{k}}
\newcommand{\corkhs}{\tilde{\hil}}
\newcommand{\cok}{\ktilde}
\newcommand{\coK}{\tilde{K}}
\newcommand{\copar}{\lambda}
\newcommand{\reg}{\gamma}
\newcommand{\rega}{\gamma_{1}}
\newcommand{\regb}{\gamma_{2}}
\newcommand{\regamb}{\gamma_{\ambient}}
\newcommand{\regintr}{\gamma_{\intrinsic}}
\newcommand{\regfn}{\Omega}
\newcommand{\regfnintr}{\regfn_{\intrinsic}}
\newcommand{\regfnamb}{\regfn_{\ambient}}
\newcommand{\regfncoreg}{\regfn_{\coreg}}
\newcommand{\bq}{\begin{equation}}
\newcommand{\eq}{\end{equation}}
\newcommand{\ba}{\begin{eqnarray}}
\newcommand{\ea}{\end{eqnarray}}

\newcommand{\spana}{\cl^{1}}
\newcommand{\spanb}{\cl^{2}}
\newcommand{\ha}{\hil^{1}}
\newcommand{\hb}{\hil^{2}}
\newcommand{\fa}{f^{1}}
\newcommand{\fb}{f^{2}}
\newcommand{\Fa}{\cf^{1}}
\newcommand{\Fb}{\cf^{2}}
\newcommand{\ka}{k^{1}}
\newcommand{\kb}{k^{2}}
\newcommand{\vcok}{\boldsymbol{\cok}}
\newcommand{\vka}{\vk^{1}}
\newcommand{\vkb}{\vk^{2}}
\newcommand{\Ka}{K^{1}}
\newcommand{\Kb}{K^{2}}
\newcommand{\kux}{\vk_{Ux}}
\newcommand{\kuz}{\vk_{Uz}}
\newcommand{\kuu}{K_{UU}}
\newcommand{\kul}{K_{UL}}
\newcommand{\klu}{K_{LU}}
\newcommand{\kll}{K_{LL}}
\newcommand{\kuxa}{\vk_{Ux}^{1}}
\newcommand{\kuza}{\vk_{Uz}^{1}}
\newcommand{\Kuua}{K_{UU}^{1}}
\newcommand{\Kula}{K_{UL}^{1}}
\newcommand{\Klua}{K_{LU}^{1}}
\newcommand{\Klla}{K_{LL}^{1}}
\newcommand{\kuxb}{\vk_{Ux}^{2}}
\newcommand{\kuzb}{\vk_{Uz}^{2}}
\newcommand{\Kuub}{K_{UU}^{2}}
\newcommand{\Kulb}{K_{UL}^{2}}
\newcommand{\Klub}{K_{LU}^{2}}
\newcommand{\Kllb}{K_{LL}^{2}}

\newcommand{\Ksum}{S}
\newcommand{\ksum}{s}
\newcommand{\vksum}{\vs}

\newcommand{\intrinsicRegMat}{M_{\intrinsic}}
\newcommand{\coregPointCloudMat}{M_{\coreg}}
\newcommand{\vid}[1]{#1^\text{vid}}
\newcommand{\aud}[1]{#1^\text{aud}}
\newcommand{\bad}[1]{#1_\text{bad}}
\newcommand{\empcompat}{\hat{\chi}}

\newcommand{\nn}{\ensuremath{k}}

\newcommand{\uva}{\ushort{\va}}
\newcommand{\uvf}{\ushort{\vf}}
\newcommand{\uvg}{\ushort{\vg}}
\newcommand{\uvk}{\ushort{\vk}}
\newcommand{\uvw}{\ushort{\vw}}
\newcommand{\uvh}{\ushort{\vh}}
\newcommand{\uvbeta}{\ushort{\vbeta}}
\newcommand{\uA}{\ushort{A}}
\newcommand{\uG}{\ushort{G}}

\def\la{{\langle}}
\def\ra{{\rangle}}
\def\R{{\reals}}
\def\psdcone{{\cs^{n}_{+}}}
\def\spectrahedron{{\psdcone(\tau)}}
\def\spectrahedrona{{\psdcone(1)}}

%




\newtheoremstyle{style}
 {\topsep}              
 {\topsep}              
 {\itshape}              
 {}                         
 {\sffamily\bfseries}  
 {.}	                     
 { }                         
 {}                          
\theoremstyle{style}

\floatname{algorithm}{\sffamily\bfseries Algorithm} 

\newcommand{\bigO}{O}
\newcommand{\X}{\mathcal{X}}
\newcommand{\Y}{\mathcal{Y}}
\newcommand{\Sym}{\mathbb{S}}
\newcommand{\Spectahedron}{\mathcal{S}}
\newcommand{\SpectahedronLeOne}{\mathcal{S}_{\le 1}} 
\newcommand{\SpectahedronLeT}{\mathcal{S}_{t}} 
\newcommand{\MaxCutPolytope}{\boxplus}
\newcommand{\FourPointPSD}{\Spectahedron^4_{\text{sparse}}}
\newcommand{\FourPointPSDplus}{\Spectahedron^{4+}_{\text{sparse}}}
\newcommand{\FourPointPSDminus}{\Spectahedron^{4-}_{\text{sparse}}}
\newcommand{\LOneBall}{\diamondsuit}
\newcommand{\signVec}{\mathbf{s}}
\newcommand{\N}{\mathbb{N}}
\newcommand{\id}{\mathbf{I}} 
\newcommand{\ind}{\mathbf{1}} 
\newcommand{\0}{\mathbf{0}} 
\newcommand{\unit}{\mathbf{e}} 
\newcommand{\one}{\mathbf{1}} 
\newcommand{\zero}{\mathbf{0}}
\newcommand\SetOf[2]{\left\{#1\,\vphantom{#2}\right.\left|\vphantom{#1}\,#2\right\}}
\newcommand{\ignore}[1]{\textbf{***begin ignore***}\\#1\textbf{***end ignore***}}

\newcommand{\todo}[1]{\marginpar[\hspace*{4.5em}\textbf{TODO}\hspace*{-4.5em}]{\textbf{TODO}}\textbf{TODO:} #1}
\newcommand{\note}[1]{\marginpar{#1}}


\newcommand{\inner}[2]{\ensuremath{\langle{#1},{#2}\rangle}}
\newcommand{\lmp}[2]{\ensuremath{\ell_{#2}^{#1}}}
\newcommand{\mapto}{\ensuremath{\rightarrow}}
\newcommand{\nmapto}{\ensuremath{\nrightarrow}}
\newcommand{\approach}{\ensuremath{\rightarrow}}
\newcommand{\imply}{\ensuremath{\Rightarrow}}
\newcommand{\inject}{\ensuremath{\hookrightarrow}}
\newcommand{\equivalent}{\ensuremath{\Longleftrightarrow}}
\newcommand{\inclusion}{\ensuremath{\hookrightarrow}}

\newcommand{\ol}[1]{\ensuremath{\overline{#1}}}

\newtheorem{definition}{Definition}
\newtheorem{proposition}{Proposition}
\newtheorem{lemma}{Lemma}
\newtheorem{corollary}{Corollary}
\newtheorem{example}{Example}
\newtheorem{remark}{Remark}
\newtheorem{theorem}{Theorem}
\newtheorem{observation}{Observation}
\newtheorem{hypothesis}{Hypothesis}
\newtheorem{notation}{Notation}

\newcommand{\A}{\mathcal{A}}
\newcommand{\B}{\mathcal{B}}

\newcommand{\G}{\mathbf{G}}
\newcommand{\D}{\mathbf{D}}
\newcommand{\bfC}{\mathbf{C}}
\newcommand{\bfW}{\mathbf{W}}
\newcommand{\bfS}{\mathbf{S}}
\newcommand{\bfD}{\mathbf{D}}
\newcommand{\bfH}{\mathbf{H}}
\newcommand{\bfA}{\mathbf{A}}
\newcommand{\bfG}{\mathbf{G}}

\newcommand{\bbC}{\mathbb{C}}

\newcommand{\bP}{\mathbb{P}}
\newcommand{\bE}{\mathbb{E}}
\newcommand{\bC}{\mathbb{C}}

\newcommand{\C}{\mathcal{C}}

\newcommand{\T}{\mathcal{T}}
\renewcommand{\O}{\mathcal{O}}

\newcommand{\Q}{\mathcal{Q}}

\renewcommand{\P}{\mathcal{P}}
\renewcommand{\S}{\mathcal{S}}

\newcommand{\F}{\mathcal{F}}
\renewcommand{\H}{\mathcal{H}}
\renewcommand{\L}{\mathcal{L}}

\newcommand{\Fre}{Fr\'echet \;}
\newcommand{\Ga}{G\^ateaux \;}

\def\diag{{\rm diag}}
\def\diam{{\rm diam}}
\def\rank{{\rm rank}}
\def\cond{{\rm cond}}

\def\supp{{\rm supp}}
\def\sinc{{\rm sinc}}

\def\argmin{{\rm argmin}}

\def\Im{{\rm Im}}

\def\proj{{\rm proj}}

\def\trace{{\rm tr}}
\def\loc{{\rm loc}}
\def\vec{{\rm vec}}
\def\nullspace{{\rm nullspace}}
\def\colspace{{\rm colspace}}
\def\rowspace{{\rm rowspace}}

\def\curl{{\rm curl}}
\def\div{{\rm div}}
\def\har{{\rm har}}

\newcommand{\h}{\mathbf{h}}
\newcommand{\x}{\mathbf{x}}
\newcommand{\s}{\mathbf{s}}
\newcommand{\w}{\mathbf{w}}
\newcommand{\z}{\mathbf{z}}
\renewcommand{\a}{\mathbf{a}}
\renewcommand{\c}{\mathbf{c}}
\renewcommand{\v}{\mathbf{v}}
\newcommand{\e}{\mathbf{e}}
\newcommand{\n}{\mathbf{n}}

\newcommand{\y}{\mathbf{y}}
\newcommand{\f}{\mathbf{f}}

\newcommand{\p}{\mathbf{p}}

\newcommand{\vecpi}{\overrightarrow{\pi}}
\newcommand{\veci}{\overrightarrow{i}}
\newcommand{\vecj}{\overrightarrow{j}}
\newcommand{\veck}{\overrightarrow{k}}

\newcommand{\1}{\mathbf{1}}

\renewcommand{\u}{\mathbf{u}}

\newcommand{\norm}[1]{\| #1 \|}
\newcommand{\comment}[1]{}
\begin{abstract}
We propose a new non-parametric framework for learning incrementally stable dynamical systems $\dot{\x} = f(\x)$ from a set of sampled trajectories.  We construct a rich family of smooth vector fields induced by certain classes of matrix-valued kernels, whose equilibria are placed exactly at a desired set of locations and whose local contraction and curvature properties at various points can be explicitly controlled using convex optimization. With curl-free kernels, our framework may also be viewed as a mechanism to learn potential fields and gradient flows. We develop large-scale techniques using randomized kernel approximations in this context. We demonstrate our approach, called contracting vector fields (CVF), on imitation learning tasks involving complex point-to-point human handwriting motions.
\end{abstract}

\IEEEpeerreviewmaketitle

\section{Introduction}

\begin{figure}[h]
    \centering
    \includegraphics[height=4.2cm, width=\linewidth]{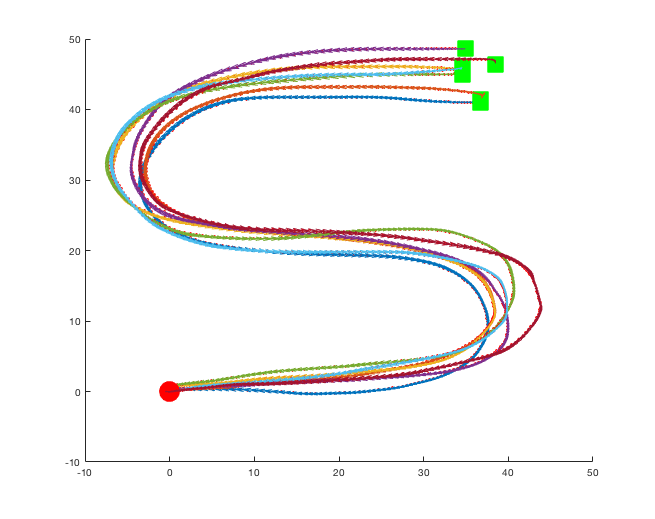}\\
    \includegraphics[height=5.0cm, width=\linewidth]{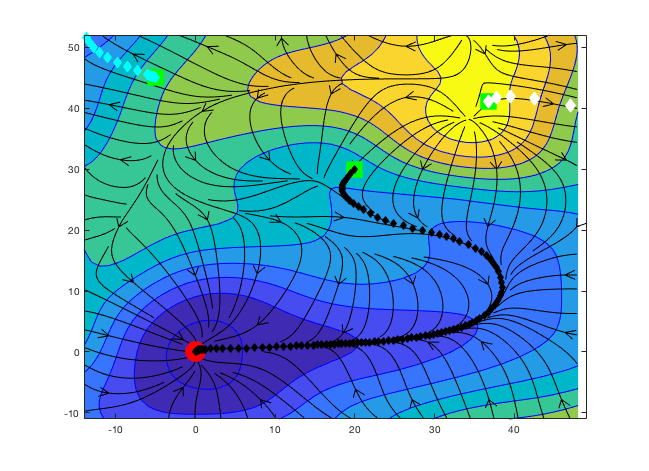}\\
    \includegraphics[height=5.0cm, width=\linewidth]{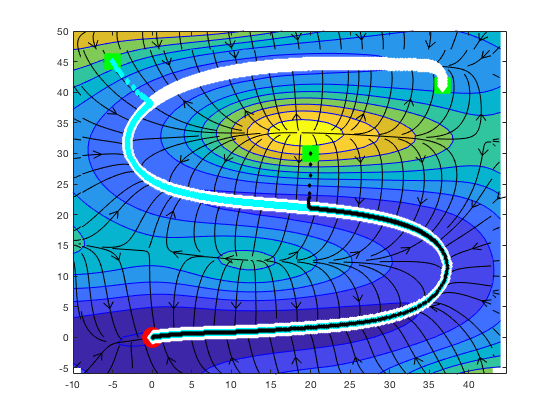}
    \caption{"S"-shape demonstration trajectories (top), naive regression-based vector-field (middle) and incrementally stable vector field learnt using our methods (bottom), with corresponding trajectories from three initial points (green squares).}\vspace{-0.25cm}
    \label{fig:start}
\end{figure}

Consider an unknown continuous-time autonomous nonlinear dynamical system evolving in $\reals^n$, $$\dot{\x} = \frac{d\x}{dt} = f(\x), ~~~\x\in \reals^n, ~~~f: \reals^n \to \reals^n$$ Starting from an initial condition $\x_0$, a trajectory $\x(t, \x_0)$ is generated by integrating the system over a time horizon. Let $\x^i_\star, i=1\ldots k$ be $k$ equilibrium points where  the induced vector field vanishes, i.e., $\dot{\x}^i_\star = f(\x^i_\star) = 0$. In this paper, we are interested in learning $f$ given desired equilibria and a small set of $N$ trajectories, $\{\x(t, \x^i_0),~~ t=0\ldots T_i,~i=1\ldots N\}$, sampled from the unknown system.

Figure~\ref{fig:start} grounds the problem stated above in an imitation learning setting. Shown in the top panel are $7$ human handwriting demonstrations recorded on a Tablet PC~\cite{lasa}. The motions start from the green points in $\reals^2$ and end at the origin (red), tracing an ``S" shape with a specific velocity profile. The learning-from-demonstrations (LfD) problem is to induce a control law from such data in order to drive a robotic system, typically the pose of an end-effector, to imitate the demonstration while reaching the goal of the motion. As succinctly summarized in~\cite{CLFDM}, modeling desired motion as the evolution of an underlying dynamical system at the kinematic level allows ``having robotic systems that have inherent adaptivity to changes in a dynamic environment, and that can swiftly adopt a new path to reach the target". This has motivated a large body of work on dynamical systems based imitation learning~\cite{khansari2011learning, khansari2017learning, harish, CLFDM, ijspeert2013dynamical}.
 
At first glance, the problem of learning a dynamical system from sampled trajectories appears to be a simple regression task, e.g., minimizing $\sum_{i, t} \|f(\x^i_t) - \dot{\x}^i_t\|^2_2$ over a suitable choice of regression models. However, a naive regression approach may be woefully inadequate, as shown in the middle panel of Figure~\ref{fig:start}: as soon as the initial conditions are even slightly different from those encountered during training, the evolution of the learnt system diverges away from the desired behavior.  The bottom panel of Figure~\ref{fig:start} shows an incrementally stable~\cite{jouffroy2010tutorial,lohmiller1998contraction} vector field learnt by methods developed in this paper: this vector field is a gradient flow on a learnt potential field; it exactly vanishes at the origin - the goal of the motion - and it sets up a region of stability around the demonstrated trajectories in order to better generalize across starting conditions and to ``pull"  perturbations encountered during execution back to the desired behavior. 

\subsection{Preview}
We give a sketch of our formulation to set the stage and introduce some notation. Given desired equilibria $Z = \{\x^i_*, i=1\ldots k\}$ and pairs $\{(\x^i_t, \dot{\x}^i_t), t=0\ldots T_i, i=1\ldots N\}$ extracted from the training trajectories, we set up the following optimization problem over a suitable non-parametric family, $\H^Z$, of vector-valued maps vanishing on $Z$,

\begin{eqnarray}
\min_{f:\reals^n\to\reals^n \in {\cal H}^Z} &\sum_{i, t} \|f(\x^i_t) - \dot{\x}_t^i\|^2_2 + \lambda \|f\|^2_{\H^Z} \label{eq:regression_objective}\\
\textrm{subject to}:&~\frac{1}{2}\Big[\mathbf{J}_f(\x^i_t) +  \mathbf{J}_f(\x^i_t)^T\Big] \preceq -\tau(\x^i_t) \mathbf{I},~\forall i, t\label{eq:stability_constraints}
\end{eqnarray}
where $\mathbf{J}_f = \frac{\partial f}{\partial \x}$ denotes the Jacobian of the vector-valued map $f$ and the notation $\mathbf{A} \preceq -\tau \mathbf{I}$ implies that the matrix $\mathbf{A}$ is negative definite with eigenvalues no larger than $-\tau$ for some $\tau>0$.  This optimization problem has the following ingredients, which we will expand on in later sections.

\begin{itemize}
\setlength{\itemsep}{1pt}
\setlength{\labelsep}{0pt}
  \setlength{\parskip}{0pt}
  \setlength{\parsep}{0pt}
\item The first term in the objective in Eqn.~\ref{eq:regression_objective} uses a least squares criterion to orient the vector field along the training trajectories. The second term controls smoothness of the vector field. $\lambda >0$ is a regularization parameter that balances these terms.
\item The constraints in Eqn~\ref{eq:stability_constraints} enforce incremental stability.  These constraints help induce a contraction tube around a nominal trajectory so that system evolution from a large set of initial conditions returns to desired behavior. In section~\ref{sec:contraction}, we provide a brief background on incremental stability and contraction analysis of dynamical systems, to motivate such constraints. 
\item The optimization problem above is solved over rich non-parametric hypothesis spaces of smooth vector-valued functions denoted by ${\cal H}$.  In particular, our construction of these hypothesis spaces is rooted in the theory of vector-valued Reproducing Kernel Hilbert Spaces (RKHS)~\citep{micchelli2005learning} which are generated by matrix-valued kernel functions.  For certain curl-free kernels~\cite{curlfree,micheli2013matrix}, the resulting vector field is actually a gradient flow. In other words, $f = -\nabla V$ for an induced smooth potential field $V$. 
\item RKHS properties can be used to construct a subspace of ${\H}$, denoted by ${\cal H}^Z$, of vector-valued functions that exactly vanish on a set of points $Z$. The optimization problem above is solved over ${\cal H}^Z$ where $Z$ is the set of desired equilibrium points. RKHS properties also imply a Representer Theorem~\cite{scholkopf2001learning} which specifies the form  of the optimal vector field, and reduces the optimization problem above to a finite dimensional convex optimization problem.
\item By using a random feature approximation to matrix-valued kernels~\cite{rahimi2008random, minh2016operator}, we are able to significantly improve training time and integration speed of the learnt dynamical system.
\end{itemize}
We view our primary contribution as bringing together the theory of vector-valued RKHS, contraction analysis and convex optimization to bear on the problem of learning stable nonlinear dynamical systems. Empirical results reported in section~\ref{sec:experiments} on a standard imitation learning benchmark confirm that our methods are competitive with several prior proposals for estimating stable dynamical systems from sampled trajectories.



 
\section{Background: Stability of Dynamical Systems}

\subsection{Notions of Stability and Lyapunov Analysis} Whether or not an explicitly given dynamical system is stable with respect to an equilibrium has been a foundational question in nonlinear control~\cite{slotine1991applied}.   A system is said to be {\it globally asymptotically stable} if solution trajectories $\x(t; \x_0)$ from any starting point $\x_0$ converge to $\x_\star$ as $t \to \infty$. The system is {\it locally asymptotically stable} if there is a ball of radius $r$ around $\x_*$ from where all initial states flow to $\x_*$.  Lyapunov's direct method~\citep{slotine1991applied} is a classical framework for verifying such stability properties of nonlinear dynamical systems.  If a suitable positive-definite scalar function, $V(\x)$, can be found that decreases along the trajectories of the system, then the evolution of the system can be thought of as continuously dissipating a generalized notion of energy,  eventually reaching an equilibrium point as a consequence -- much like a ball rolling down a mountainous landscape to the deepest point in a valley. In mathematical terms, energy dissipation is stated as follows: if a dynamical system $\dot{\x} = f(\x)$ can be associated with a function $V(\x)$ that has a local or global minimum at ${\x}^\star$ and whose time derivative is negative everywhere or in the vicinity of $\x^\star$, i.e., 
\begin{equation}
\dot{V}(\x) = \frac{dV(\x(t))}{dt} = \nabla V(\x)^T f(\x)  <  0\label{eq:energy_decrease}
\end{equation} then the system is certified to be locally or globally stable respectively. Converse Lyapunov theorems prove existence of Lyapunov functions for stable systems~\citep{slotine1991applied}, but despite these existence results, Lyapunov theory is largely unconstructive even when the dynamics is explicit; it does not prescribe how to find Lyapunov functions for verifying the stability of a given general nonlinear system. For a review of various methods, we refer the reader to~\cite{giesl2015review}.  A few special cases where the construction is well-understood are worth mentioning. Stable linear systems admit quadratic Lyapunov functions that can be found via via linear algebraic techniques. If a polynomial dynamical system admits a polynomial Lyapunov function, then one can search for it using sum-of-squares techniques~\citep{ahmadi2012algebraic} which reduce to instances of semidefinite programming (SDP).  However, it is also known that there exist stable dynamical systems for which no polynomial Lyapunov function, or sum-of-squares Lyapunov certificate exists~\citep{ahmadi2012algebraic}.

\subsection{Incremental Stability and Contraction Analysis} \label{sec:contraction} Stronger notions of stability called {\it incremental stability} and associated contraction analysis tools~\cite{jouffroy2010tutorial,lohmiller1998contraction} are concerned with the convergence of system trajectories with respect to each other, as opposed to stability with respect to a single single equilibrium. Contraction analysis derives sufficient conditions under which the displacement between any two trajectories $\x(t; \x_0)$ and $\x(t; \x_1)$ starting from initial conditions $\x_0, \x_1$ will go to zero. If $f$ is continuously differentiable, then $\dot{\x} = f(\x)$ implies the differential relation,
$$
\dot{\delta \x} = \mathbf{J}_f(\x) {\delta \x}~~~\textrm{where}~~\mathbf{J}_f = \frac{\partial f}{\partial \x}
$$ The object $\delta \x$, referred to as virtual displacement, is to be thought of as infinitesimal spatial displacement between neighboring trajectories at a fixed time. The rate of change of the corresponding infinitesimal squared distance, $\delta\x^T\delta\x$, can be expressed as,
$$
\frac{d}{dt} (\delta\x^T \delta\x) = 2 \delta\x^T \dot{\delta\x} = \delta\x^T \mathbf{J}_f(\x) \delta\x
$$
Hence, if the symmetric part of Jacobian of $f$ at $\x$ is negative definite, then the distance between neighboring trajectories shrinks. In particular, if the following condition, which inspires Eqn.~\ref{eq:stability_constraints}, holds for some smooth $\tau(x) > 0$,
\begin{equation}
\frac{1}{2}\Big[\mathbf{J}_f(\x) + \mathbf{J}_f(\x)^T\Big] \preceq -\tau(\x) \mathbf{I}\label{eq:contraction}
\end{equation} then the following is implied,
$$
\frac{d}{dt} (\delta\x^T \delta\x)  \leq -2\tau(\x) \delta\x^T \delta\x
$$  Integrating both sides yields,
$$
\|\delta\x_t\|^2_2 \leq \|\delta\x_0\| e^{-\int_{0}^t \tau(\x) dt}
$$
Hence, any infinitesimal length $\|\delta\x\|$ converges exponentially to zero as time goes to infinity. This implies that in a contraction region, i.e., the set of $\x$'s where Eqn.~\ref{eq:contraction} holds, trajectories will tend to together converge towards a nominal path. If the entire state-space is contracting and a finite equilibrium exists,
then this equilibrium is unique and all trajectories converge
to this equilibrium.
 
While this exposition suffices for our needs in this paper, it should be noted that contraction theory~\cite{lohmiller1998contraction} more broadly considers generalized distances of the form $\delta\x^T \mathbf{M}(\x)\delta\x$ induced by a symmetric, positive definite matrix function $\mathbf{M}(\x)$. The search for a contraction metric may be interpreted as the search for a Lyapunov function of the specific form $V(\x) = f(\x)^T \mathbf{M}(x) f(\x)$.  As is the case with Lyapunov analysis in general, finding such an incremental stability certificate for a given dynamical system is a nontrivial problem; see~\cite{aylward2008stability} and references therein. 

\subsection{Prior Work in Learning Stable Dynamical Systems}

In a dynamical systems approach to feedback control, robot motion during a task (for example reaching a cup) is formulated as a differential equation rather than a time-indexed trajectory. 
Compared to classical approaches based on following a time-indexed trajectory, such representation allows better generalization since instead of memorizing the demonstration trajectories, the policy has to capture the essential dynamics underlying the task during training. Additionally a dynamical systems policy can, by construction, adapt to changes in dynamic environments, making it suitable for use in unstructured environments \cite{billard_08}. These properties have paved the way for dynamical systems policy to be widely used for robot learning. In this section we provide a brief overview of three dynamical systems methods which are used for comparison in our experimental study. We refer interested readers to \cite{ijspeert2013dynamical, khansari2017learning, lemme2015open, argall_survey_09, billard_08} for a more thorough overview of works in this field.

Dynamic Movement Primitives (DMPs) are the most widely used dynamical systems approaches which have been used both for imitation learning and reinforcement learning \cite{ijspeert2013dynamical}. The dynamical system defined by DMP is composed of two main terms: a nonlinear term to accurately encode a given demonstration, and a linear term that acts as a PD controller. These two terms are coupled through a phase variable. Global stability is ensured by smoothly switching from the non-linear term to the stable linear term via the phase variable. The phase variable in a DMP make it a time varying system which depending on the application may make the system sensitive to perturbations. In addition, DMPs can only be trained from one demonstration one degree-of-freedom at a time, and hence they do not directly benefit from multiple training demonstrations with correlated dimensions.

Stable Estimator of Dynamical Systems (SEDS) \cite{khansari2011learning} is another widely used approach for learning a nonlinear dynamical systems from a set of demonstrations. SEDS uses a Guassian mixture model to represent the policy and imposes constraints on means and covariance of Guassian mixtures to ensure Global asymptotic stability of the trained model. The stability criteria in SEDS is derived based on a simple quadratic Lyapunov function. SEDS can only model trajectories whose distances to the target decrease monotonically in time.

Control Lyapunov Function-based Dynamic Movements (CLF-DM) \cite{CLFDM} is another approach that is inspired from control theory to stabilize a learned dynamical systems. CLF-DM learns a parametric Lyapunov function from a set of given demonstrations. It then uses any of the-state-of-the art regression techniques to learn an (unstable) dynamical systems from the demonstrations. Finally it uses the learned control Lyapunov function to derive a control command to stabilize the learned (unstable) dynamical systems.

SEDS and CLFDM involve non-convex optimization for dynamics fitting and constructing Lyapunov functions respectively, and are hence prone to sub-optimal local minima.


\section{Learning Contracting Vector Fields}

The problem of estimating smooth vector fields in $\reals^n$ can be naturally formulated in terms of Tikhonov regularization in a {\it vector-valued} Reproducing Kernel Hilbert Space (RKHS)~\citep{micchelli2005learning, alvarez2012kernels}.  The theory and formalism of vector-valued RKHS can be traced as far back as the
work of Laurent Schwarz in 1964~\citep{schwartz} with applications ranging from solving partial differential equations to machine learning~\citep{sindhwani2012scalable}. They may be viewed as a systematic generalization of scalar kernel methods~\citep{scholkopf2001learning} more familiar in machine learning. 

\subsection{Vector Fields generated by Matrix-valued Kernels}

To be an RKHS, any Hilbert Space $\H$ of vector fields in $\reals^n$ must satisfy a natural continuity criteria as given in the definition below. 

\begin{definition}
We say that $\H$ is a Reproducing Kernel Hilbert Space of vector fields in $\reals^n$ if for any $\v\in \reals^n$, the linear functional that maps $f\in \H$ to $\v^T f(\x)$ is continuous.
\end{definition}

Any RKHS vector field over $\reals^n$ can be associated with   a {\it matrix-valued} kernel function $K: \reals^n \times \reals^n \mapsto \reals^{n\times n}$. In other words, for any inputs $\x, \y$ in $\reals^n$, $K(\x, \y)$ returns an $n\times n$ matrix. Valid kernel functions are {\it positive} in the specific sense that for any finite set of points $\{\x_i \in \reals^n \}_{i=1}^{l}$, the $ln \times ln$ Gram matrix of $K$ defined by the $n\times n$ blocks, $\G_{ij} = K(\x_i, \x_j),~1\leq i,j \leq l$,  is positive definite. We have the following characterization.

\begin{definition}
A Hilbert space $\H$ of vector-valued functions mapping $\reals^n \rightarrow \reals^n$, with inner product denoted by $\la \cdot, \cdot \ra_{\H}$, is a Reproducing Kernel Hilbert Space (RKHS) if there is a positive matrix-valued function $K: \reals^n \times \reals^n \rightarrow \reals^{n\times n}$ such that for all $\x, \y \in \reals^n$,
\begin{enumerate}
\item The vector-valued map, $K(\cdot, \x) \y$ defined by $\z \to K(\z, \x) \y \in \H$.
\item For all $f \in \H$, the \textbf{reproducing property} holds
\begin{equation}\label{equation:reproducing2}
\la f, K(., \x) \y\ra_{\H} =  f(\x)^T \y
\end{equation}
\end{enumerate}
$K$ is called the reproducing kernel for  $\H$.
\end{definition} 

Conversely, any kernel $K$ uniquely determines an RKHS which admits $K$ as the reproducing kernel. This RKHS, denoted by $\H_K$, is defined to be the completion of the linear span of functions $\{ K(\cdot, \x) \y,~~\x,\y\in \reals^n\}$ with inner product given by, $\la \sum_i K(\cdot, \x_i) \boldsymbol{\alpha}_i, \sum_j K(\cdot, \z_j) \boldsymbol{\beta}_j \ra_{\H_K} = \sum_{i,j} \boldsymbol{\alpha}_i^T K(\x_i, z_j) \boldsymbol{\beta}_j$.

Due to the reproducing property, as in the scalar case, standard learning problems in a vector-valued RKHS can be turned into finite dimensional optimization problems using a natural matrix-vector generalization of the classical Representer theorem~\cite{scholkopf2001learning}.
%

%
\begin{theorem}[Representer Theorem]
The optimal solution to any vector field learning problem of the form,
$$
f^* = \argmin_{f\in \H_K} L(f(\x_1)\ldots  f(\x_l)) + \lambda \|f\|^2_{{\H}_K}, 
$$ is a sum of matrix-vector products of the form,
\begin{equation}
f^*(\x) = \sum_{i=1}^l K(\x, \x_i) \boldsymbol{\alpha}_i~~\label{eq:representer}
\end{equation} where $\boldsymbol{\alpha}_i \in \reals^n,~~i=1\ldots l$,  $L$ is an arbitrary loss function (which can also be an indicator function encoding arbitrary constraints) and $\lambda>0$ is a regularization parameter.
\end{theorem} 
When the learning problem involves Jacobian evaluations, as in the main optimization problem of interest in Eqn.~\ref{eq:regression_objective}-\ref{eq:stability_constraints}, we need an extended Representer Theorem along the lines of Theorem 1 in~\cite{zhou2008derivative}.
\begin{theorem}[Form of Optimal Contracting RKHS Vector Field]
The optimal solution to any vector field learning problem of the following form (includes Eqn.~\ref{eq:regression_objective}-\ref{eq:stability_constraints}),
$$
f^* = \argmin_{f\in \H} L(f(\x_1)\ldots  f(\x_l); J_f(\x'_1)\ldots J_f(\x'_m)) + \lambda \|f\|^2_{{\H}_K}, 
$$ is a sum of matrix-vector products of the form,
\begin{equation}
f^*(\x) = \sum_{i=1}^l K(\x, \x_i) \boldsymbol{\alpha}_i + \sum_{j=1}^m \sum_{k=1}^n \frac{\partial K(\x, \x'_j)}{\partial x_j} \boldsymbol{\beta}_{ik} ~~\label{eq:representer3}
\end{equation} where $\boldsymbol{\alpha}_i, \boldsymbol{\beta}_{ik}  \in \reals^n$,  $L$ is an arbitrary loss function (which can also be an indicator function encoding arbitrary constraints) and $\lambda>0$ is a regularization parameter.
\end{theorem} 
Eqn.~\ref{eq:representer3} implies that the optimization problem in Eqn.~\ref{eq:regression_objective}-\ref{eq:stability_constraints} can be reduced to a finite dimensional regression problem involving Linear Matrix Inequalities (LMI) over the variables $\boldsymbol{\alpha}_i, \boldsymbol{\beta}_{ik}$. In section~\ref{sec:random_embeddings}, we use randomized low-rank approximations to the kernel function to develop a scalable solver.   

\subsection{Choice of Matrix-valued Kernels}
In this paper, we will consider two choices of matrix valued kernels:
\begin{itemize}
    \item {\it Gaussian Separable Kernels}, $K_\sigma$, defined by the scalar Gaussian kernel $k_{\sigma}(x, y) = e^{-\frac{\|\x-\y\|^2_2}{2\sigma^2}}$ times the $n\times n$ identity matrix, 
    \begin{equation}
    K_\sigma(\x, \y) = k_{\sigma}(x, y)\mathbf{I}\label{eq:gaussian_separable}
    \end{equation} For this choice, each individual component of the vector field $f=(f_1\ldots f_n)$ belongs to the scalar RKHS ${\cal H}_k$ associated with the standard Gaussian kernel. More generally, one may consider separable matrix-valued kernels~\cite{sindhwani2012scalable} of the form $K(\x, \y) = e^{-\frac{\|\x-\y\|^2_2}{2\sigma^2}} \mathbf{L}$ for a positive definite $n\times n$ matrix $\mathbf{L}$.
    \item {\it Curl-free Kernels}~\cite{micheli2013matrix, curlfree} are defined by the Hessian of the scalar Gaussian kernel,
    \begin{equation}
    K_{cf}(\x, \y) = \frac{1}{\sigma^2} e^{\frac{-\|\x - \y \|^2}{2\sigma^2}} \left[\boldsymbol{I} - \frac{(\x - \y) (\x - \y)^T}{\sigma^2} \right] \label{eq:curlfree_gaussian}
    \end{equation}
    This choice is interesting because vector fields in the associated RKHS are curl-free and can be interpreted as gradient flows with respect to a potential field $V$, i.e.  $$\dot{\x} = f(\x) = -\nabla V(\x).$$ Consequently, the Jacobian of $f$, $\boldsymbol{J}_f = -\nabla^2 V$, at any $\x$ is symmetric being the Hessian of $-V$.  Following~\cite{micheli2013matrix}, we derive a formula for $V$ in the following proposition,
    \begin{prop} Let $f\in {\cal H}_{K_{cf}}$ have the form,
$$
f(\x) = \sum_{i=1}^l K_{cf}(\x, \x_i) \boldsymbol{\alpha}_i.
$$ Then, $f(\x) = \nabla V(\x)$ where $V: \reals^n \mapsto \reals$ has the form,
\begin{equation}
V(\x) = -\sum_{i=1}^l \nabla_{\x} k_\sigma(\x, \x_i)^T \boldsymbol{\alpha}_i\label{eq:potential}
\end{equation} 
\end{prop}
\end{itemize}

\subsection{Subspace of RKHS Vector Fields that Vanish on a Point Set}\label{subsec:vanishing_rkhs}
We are now interested in constructing a family of vector fields that vanish at desired points; these points are desired equilibria of the dynamical system we wish the learn. Let $Z = \{\x^*_1, \ldots, \x^*_p\}$ be a set of points. Given ${\cal H}_K$, consider the subset of functions that vanish on $Z$, $${\cal H}^Z_K = \{f \in {\cal H}_K : f(\x^*_i) = 0 \in \reals^n, \x^*_i\in Z\}.$$  Infact, ${\cal H}^Z_K$ is a closed subspace of ${\cal H}_K$ and itself an RKHS associated a modified kernel function $K^Z$. We have the following result, a simple generalization of Theorem 116 in~\cite{berlinet2011reproducing}.
\begin{prop} ${\cal H}^Z_K \subseteq {\cal H}_K$ is an RKHS whose matrix-valued kernel is given by,
\begin{equation}
K^{Z}(\x, \y) = K(\x, \y) - K(\x, Z) K(Z, Z)^{-1} K(Z, \y)
\end{equation}
\end{prop}
 
 Above we use the following notation: given any two sets of points $S = \{\x_i \in \reals^n \}_{i=1}^{l_{1}}$ and $S' = \{\y_i \in \reals^n \}_{i=1}^{l'}$, the Gram matrix of any matrix-valued kernel $K$ on $S, S'$, denoted by $K(S', S)$ is the $l'n \times ln$ matrix defined by the $n\times n$ blocks, $\G_{ij} = K(\y_i, \x_j) \in \reals^{n\times n}$.
 
Hence, we can start with any base matrix-valued kernel $K$, define $K^Z$ as above and use its associated RKHS as a space of vector fields that are guaranteed to vanish on $Z$, the desired set of equilibrium points.

\subsection{Faster Solutions using Random Feature Approximations}
\label{sec:random_embeddings}
The size of the problem using the full kernel expansion in Eqn.~\ref{eq:representer3} grows as $ln$, the number of demonstration data points times the dimensionality of the problem. This makes training slow for even moderately long demonstrations even in low-dimensional settings. More seriously, the learnt dynamical system is slow to evaluate and integrate at inference time. We now develop a practical solver using random feature approximations to kernel functions, that have been extensively used to scale up training complexity and inference speed of kernel methods~\cite{rahimi2008random, huang2014kernel} in a number of applications. These approximations have only recently been extended to matrix-valued kernels~\cite{minh2016operator,brault2016random}. 

Given a matrix-valued kernel $K$, the basic construction starts by defining a matrix-valued feature map $\Phi: \reals^n \to \reals^{D \times n}$ having the property that,
$$
K(\x, \y) \approx \Phi(\x)^T \Phi(\y)
$$ where $D$ controls the quality of the approximation. First note that armed with such an approximation one can reparameterize vector-valued RKHS maps as follows,
$$
f(\x) = \sum_{i=1}^l K(\x, \x_i) \boldsymbol{\alpha}_i \approx \sum_{i=1}^l \Phi(\x)^T \Phi(\x_i) \boldsymbol{\alpha_i} =  \Phi(\x)^T \theta,
$$  where $\theta = \sum_{i=1}^l \Phi(\x_i) \boldsymbol{\alpha_i} \in \reals^{D}$. Thus, instead of optimizing $ln$ variables $\{\boldsymbol{\alpha}_i \in \reals^n, i=1\ldots l\}$, we only need to optimize $D$ variables $\theta$. The choice of $D$ depends on the quality-time tradeoffs demanded by an application. We now define feature maps for approximating the kernels in Eqns.~\ref{eq:gaussian_separable} and~\ref{eq:curlfree_gaussian}.

\subsubsection{Matrix-valued Random Feature Maps for Gaussian Separable and Curl-free Kernels}

It is well known~\citep{rahimi2008random} that the random scalar feature map $\phi(\x) : \reals^n \longrightarrow \reals^s$ as
\begin{equation}
        \phi(\x) = \sqrt{\frac{2}{s}} \begin{bmatrix} \cos(\w_1^T \x + b_1) \\ \vdots \\ \cos(\w_s^T \x + b_s) \end{bmatrix} \:.
\end{equation}
where $\w_1, ..., \w_s$ are i.i.d. draws from ${\cal N}(0, \sigma^{-2} {\mathbf I})$, and $b_1, ..., b_s$ are i.i.d. draws from $\mathrm{Unif}[0, 2\pi]$, induces a low-rank approximation to the Gaussian kernel (with bandwidth $\sigma$). Other shift invariant kernels also admit such approximations. With some calculations, this immediately implies   matrix-valued feature map approximations, $$K_\sigma(\x, \y) \approx \Phi_{\sigma}(\x)^T \Phi_{\sigma}(\y),  K_{cf}(\x, \y) \approx \Phi_{cf}(\x)^T \Phi_{cf}(\y),$$ for the Gaussian Separable Kernel (Eqn.~\ref{eq:gaussian_separable}) and the Curl-free kernels (Eqn.~\ref{eq:curlfree_gaussian}) respectively whose formulae are given below,  

\begin{eqnarray}
\Phi_{\sigma}(\x) &=& \phi(\x) \otimes \mathbf{I} \label{eq:random_features_gs} \\
\Phi_{cf}(\x) &=& \sqrt{\frac{2}{D}} \begin{bmatrix} \sin(\w_1^T \x + b_1)\w_1^T \\ \vdots \\ \sin(\w_D^T \x + b_D) \w_D^T \end{bmatrix} \label{eq:random_features_cf}
\end{eqnarray} where $\otimes$ denotes Kronecker product.

\subsubsection{Random Features Vanishing on a Point Set}
 
In section~\ref{subsec:vanishing_rkhs}, we gave a recipe to go from a kernel $K$ to $K^Z$ in order to generate a subspace of vector fields that vanish on a set of desired equlibrium points $Z$. Analogously, in this section we define a procedure to a matrix-valued feature map $\Phi$ to $\Phi^Z$ such that $\Phi^{Z}(\x)$ vanishes on $Z$. 
 
For a set of points $X = (\x_1, ..., \x_l)$,
define
$$
\Phi(X) = \begin{bmatrix} \Phi(\x_1), ..., \Phi(\x_{X}) \end{bmatrix} \in \R^{D \times n l} 
$$
Since $K(\x, \y) \approx \Phi(\x)^T \Phi(\y)$ we have,
{\scriptsize 
\begin{align*}
        K^Z(\x, \y) &= K(\x, \y) - K(\x, Z) K(Z, Z)^{-1} K(Z, \y) \\
        &= \Phi(\x)^T \Phi(\y)-\Phi(\x)^T \Phi(Z) (\Phi(Z)^T \Phi(Z))^{-1} \Phi(Z)^T \Phi(\y) \\
        &= \Phi(\x)^T [ I - \Phi(Z) (\Phi(Z)^T \Phi(Z))^{-1} \Phi(Z)^T ] \Phi(\y) \\
        &= \Phi(\x)^T [ I - P_{\Phi(Z)} ] \Phi(\y) \\
        &= \Phi(\x)^T P^\perp_{\Phi(Z)} \Phi(\y) \:.
\end{align*}}
Above, $P_M$ denotes the orthogonal projector onto the range of $M$.
We can write $P^\perp_{\Phi(Z)} = LL^T$ for some $L \in \R^{D \times D}$.
Hence, we now define a new feature map as
\begin{align}
        \Phi^{Z}(\x) = L^T \Phi(\x) \:,
\end{align}
which satisfies the property that $K^Z(\x, \y) = \Phi^{Z}(\x)^T \Phi^{Z}(\y)$.
Note that despite the fact that the kernel $K^Z(\x, \y)$ is not
shift-invariant, this particular construction inherits
the ability to be expressed as a low-rank feature map while guaranteeing that $\Phi^{Z}(\x)$ vanishes on $Z$.
 
\subsection{Regression with LMI Constraints}
\label{sec:lmi_solver}
Using matrix-valued random feature approximation to kernels, the vector field we seek to learn has the form,
\begin{equation}
\dot{\x} = \Phi^Z(\x)^T \theta = \sum_{i=1}^D \Phi^Z_i(\x) \theta_i
\end{equation} where $\Phi^Z(\x)^T = [\Phi_1^Z(\x)\ldots \Phi_d^Z(\x)],~~\Phi^T_i: \reals^n \mapsto \reals^n$. Let $\boldsymbol{J}_{\Phi^Z_i}$ denote the $n\times n$ Jacobian matrix of $\Phi^Z_i$. Then, the optimization problem in Eqns.~\ref{eq:regression_objective}-\ref{eq:stability_constraints} reduces to,
\begin{eqnarray}
\min_{\theta\in \reals^D} \sum_{i, t} \|\Phi^Z(\x^i_t) - \dot{\x}^i_t\|^2_2 &+& \lambda \|\theta\|^2_2\\
\textrm{subject to}:~~~~~~~~~~~~~~~~~~~~~~~~&&\nonumber\\
\frac{1}{2} \sum_{j=1}^D \Big[\boldsymbol{J}_{\Phi^Z_j}(\x^i_t) + \boldsymbol{J}^T_{\Phi^Z_j}(\x^i_t)\Big] \theta_j &\preceq& -\tau(\x^i_t) \boldsymbol{I}\label{eq:stability_constraints2}
\end{eqnarray}

We solve the regression problem with LMI constraints above using an  ADMM-based first order method for large-scale convex cone programs, implemented in Splitting Conic Solver (SCS)~\cite{o2016conic} (with its backend configured to use direct linear system solvers). Note that the contraction constraints in Eqn.~\ref{eq:stability_constraints2} may be enforced only a subsample of points. Slack variables may be added to ensure feasibility.

%


\section{Empirical Analysis}\label{sec:experiments}

\subsection{Imitating Human Handwriting Motions}
We evaluate our methods on the LASA library of two-dimensional human handwriting motions commonly used for benchmarking dynamical systems based movement generation techniques in imitation learning settings~\cite{khansari2017learning, lemme2015open, harish}. This dataset contains $30$ handwriting motions recorded with a pen input on a Tablet PC.  For each motion, the user was asked to draw 7  demonstrations of a desired pattern, by starting from different initial positions and ending at the same final point. Each demonstration trajectory comprises of $1000$ position ($\x$) and velocity ($\dot{\x}$) measurements.  We report comparisons on a subset of $4$ shapes: {\it Angle, CShape, GShape} and {\it JShape$_2$} shown in Figure~\ref{fig:datasets}, together with statistics on average speed (s mm per second), movement duration (T seconds) and position (pos-dev) and velocity (speed-dev) deviation about the average of human demonstrations (reported in the title of the plots).

\begin{figure}[h]
    \centering
    \includegraphics[height=8cm, width=\linewidth]{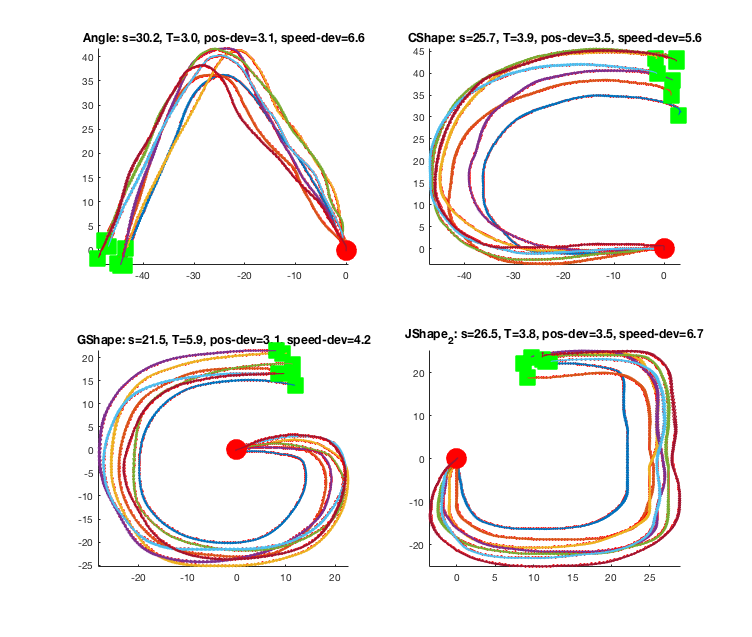}
    \caption{7 demonstrations for 4 shapes}.
    \label{fig:datasets}
\end{figure}

\subsection{Qualitative Results}
In Figure~\ref{fig:contraction_tube}, we show the region of contraction associated with a contracting vector field (curl-free) learnt on a sampled version of ``S" shape (Figure~\ref{fig:start}) with $\tau=100$.  The contours correspond to the largest eigenvalue of the symmetrized Jacobian of the learnt vector field in a grid around the demonstrations.

\begin{figure}[h]
\includegraphics[height=5cm,width=\linewidth]{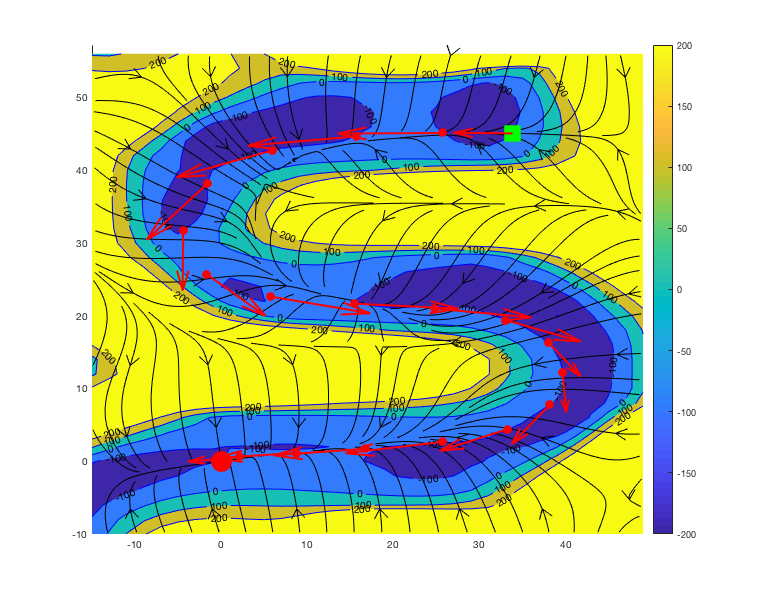}
\caption{Contracting Vector Field $\dot{\x} = f(\x)$ and associated contraction tube learnt on ``S"-shape data; contours correspond to the largest eigenvalue of the Jacobian of $f$.}\label{fig:contraction_tube}
\end{figure}

\begin{figure*}[t]
    \centering
     \includegraphics[height=4cm, width=0.24\linewidth]{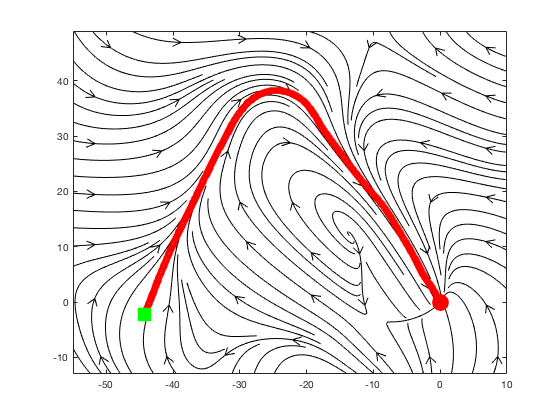}
     \includegraphics[height=4cm, width=0.24\linewidth]{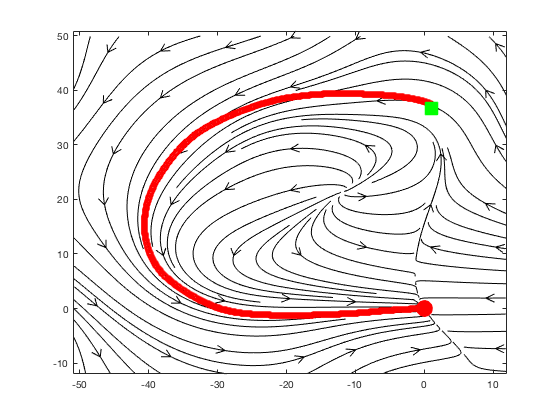}
     \includegraphics[height=4cm, width=0.24\linewidth]{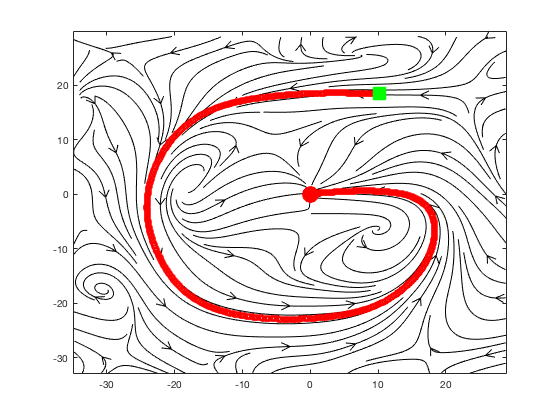}
    \includegraphics[height=4cm, width=0.24\linewidth]{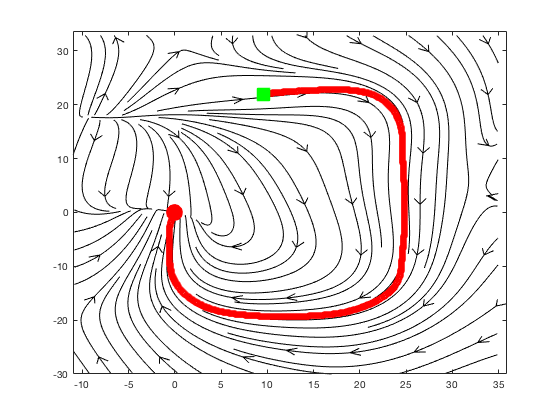}\\
    \includegraphics[height=4cm, width=0.24\linewidth]{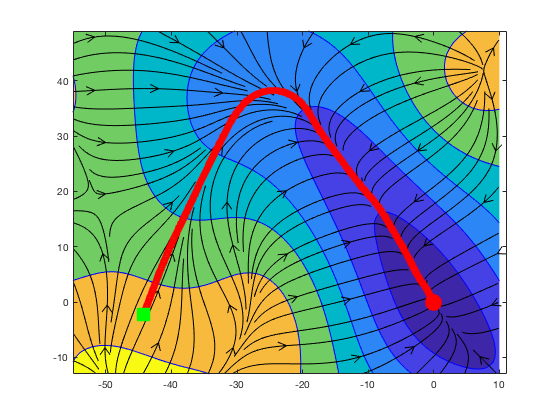}
     \includegraphics[height=4cm, width=0.24\linewidth]{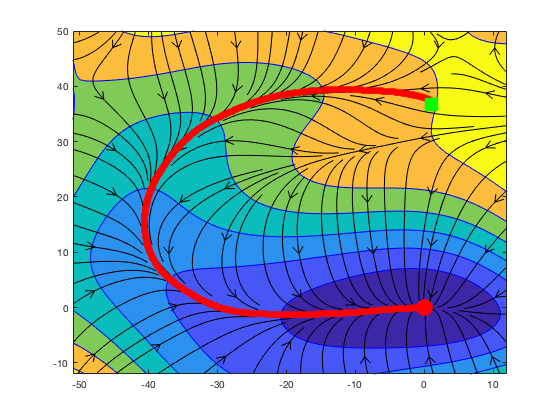}
     \includegraphics[height=4cm, width=0.24\linewidth]{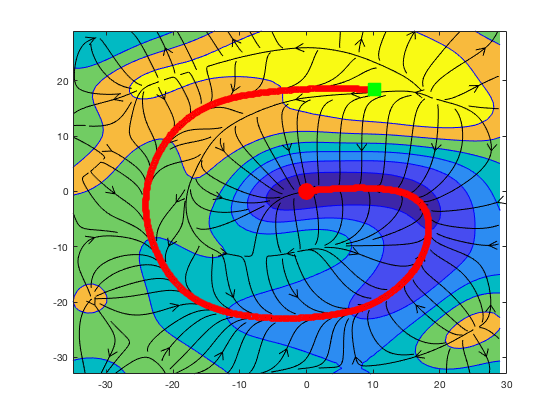}
    \includegraphics[height=4cm, width=0.24\linewidth]{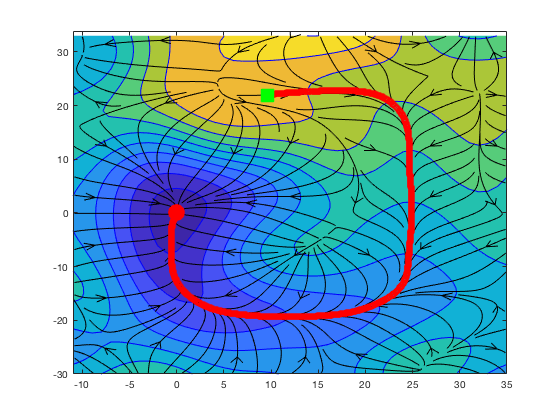}
    \caption{Contracting Vector Fields learnt for Gaussian Separable (top), and Curl-free random feature maps (bottom). For the latter, the contours correspond to the associated potential field and streamlines are gradient flows.}
    \label{fig:cvfs}
\end{figure*}

Figures 3 and 4 show vector fields learnt by our methods on J and Angle shapes. It can be seen that the Gaussian separable and curl-free kernels induce vector fields representing qualitatively different dynamics. In both cases, one can see that sizeable contraction tubes are setup around the demonstrations. For the curl-free kernel, we can also compute a potential field so that the vector field may be interpreted as a gradient flow (Eqn.~\ref{eq:potential}) with respect to it. 

\subsection{Comparison Metrics}
We now describe how we can compare the efficacy of different dynamical systems for imitation learning tasks. Each trained dynamical system, $\dot{\x} = f(\x)$ is integrated from a given starting condition for either a certain time horizon or until the event that the state $\x$ is sufficiently close to the goal. We use a high quality integrator\footnote{\url{https://www.mathworks.com/help/matlab/ref/ode45.html}} to ensure integration errors do not influence the comparisons. The trajectories generated by the dynamical system are evaluated with respect to three broad criteria: 
\begin{compactitem}
\item {\bf Reproduction Accuracy}: How well does the dynamical system reproduce positions and velocities in training and test demonstrations, when started from same initial conditions and integrated for the same amount of time as the human movement duration ($T$).   Specifically, we measure reproduction error with respect to $m$ demonstration trajectories as,
\begin{eqnarray}
\textrm{TrajectoryError} = \frac{1}{m}\sum_{i=1}^m\frac{1}{T_i}\sum_{t=0}^{T_i} \|\x^i_t - \hat{\x}^i_t\|_2\label{eq:traj_err}\\
\textrm{VelocityError} = \frac{1}{m}\sum_{i=1}^m\frac{1}{T_i}\sum_{t=0}^{T_i} \|\dot{\x}^i_t - \hat{\dot{\x}}^i_t\|_2\label{eq:vel_err}
\end{eqnarray}
The metrics {\it TrainingTrajectoryError, TestTrajectoryError, TrainingVelocityError, TestVelocityError} report these measures with respect to training and test demonstrations. At the end of the integration duration ($T$), we also report {\it DistanceToGoal}: how far the final state is from the goal (origin).  Finally, to account for the situation where the learnt dynamics is somewhat slower than the human demonstration, we also generate trajectories for a much longer time horizon ($30T$) and report ${\it DurationToGoal}$: the time it took for the state to enter a ball of radius $1mm$ around the goal, and how often this happened for the $7$ demonstrations ({\it NumReachedGoal}). 

\item {\bf Stability}:  To measure stability properties, we evolve the dynamical system from $16$ random positions on a grid enclosing the demonstrations for a long integration time horizon (30T). We report the fraction of trajectories that reach the goal ($GridFraction$); the mean duration to reach the goal when that happens ($GridDuration$); the mean distance to the Goal ($GridDistanceToGoal$) and the closest proximity of the generated trajectories to a human demonstration, as measured using Dynamic Time Warping Distance ($GridDTWD$)~\cite{keogh2005exact} (since in this case trajectories are likely of lengths different from demonstrations).
\item {\bf Training and Integration Speed}: We measure both training time as well as time to evaluate the dynamical system which translates to integration speed.
\end{compactitem}

\subsection{Comparison with DMP, SEDS and CLFDM}
{\small
\setlength{\tabcolsep}{0.8pt}

\begin{table*}[t]                                                    \begin{minipage}{.5\linewidth}
\centering                                                            
\begin{tabular}{|c|c|c|c H H}                                        
\hline                                                                
 & DMP & SEDS & CLFDM & CVF-CF & CVF-GS \\                  
\hline                                                                
TrainingTrajectoryError & 4.10 & 7.18 & 4.93 & 6.70 & 5.84 \\         
\hline                                                                
TrainingVelocityError & 7.40 & 14.62 & 10.99 & 17.70 & 15.25 \\       
\hline                                                                
TestTrajectoryError & 5.53 & 4.58 & 12.24 & 6.09 & 4.16 \\            
\hline                                                                
TestVelocityError & 8.74 & 11.43 & 15.50 & 15.50 & 13.01 \\           
\hline                                                                
DistanceToGoal & 3.59 & 3.25 & 6.70 & 6.65 & 3.99 \\                  
\hline                                                                
DurationToGoal & - & 3.89 & 4.34 & 4.76 & 3.50 \\                     
\hline                                                                
NumberReachedGoal & 0.00 & 7.00 & 7.00 & 7.00 & 7.00 \\               
\hline                                                                
GridDuration & 5.91 & 3.72 & 9.73 & 3.40 & 2.41 \\                    
\hline                                                                
GridFractionReachedGoal & 0.06 & 1.00 & 1.00 & 1.00 & 1.00 \\         
\hline                                                                
GridDistanceToGoal & 3.33 & 1.00 & 1.00 & 1.00 & 1.00 \\              
\hline               
GridDTWD & 24493.27 & 13865.75 & 14503.13 & 11311.29 & 8865.90 \\     
\hline  
TrainingTime & 0.05 & 2.10 & 2.82 & 7.58 & 3.29 \\ 
\hline
IntegrationSpeed & 0.21 & 0.06 & 0.15 & 0.15 & 0.04 \\ 
\hline
\end{tabular}%
\begin{tabular}{|c|c|}
     \hline
     CVF-CF & CVF-GS \\
     \hline
     5.44 & 5.21 \\
     \hline
     11.92 & 10.45 \\
     \hline
     4.17 & 3.19 \\
     \hline
     11.27 & 9.05 \\
     \hline
     3.69 & 2.04 \\
     \hline
     3.29 & 3.01 \\
     \hline
     7.00 & 7.00 \\
     \hline
     1.73 & 2.30 \\
     \hline
     1.00 & 1.00 \\
     \hline
     1.00 & 1.00 \\
     \hline
     12837.89 & 9451.88 \\
     \hline
     23.64 & 23.90 \\
     \hline
     0.04 & 0.04 \\
     \hline
     \end{tabular}
\caption{Angle}                                                       
\label{table:Angle}                                                   \end{minipage}
\begin{minipage}{.5\linewidth}
\centering                                                      
\begin{tabular}{|c|c|c|cHH}                                          
\hline                                                                  
 & DMP & SEDS & CLFDM & CVF-CF & CVF-GS \\                    
\hline                                                                  
TrainingTrajectoryError & 4.40 & 8.29 & 6.90 & 8.83 & 10.62 \\         
\hline                                                                  
TrainingVelocityError & 7.32 & 14.73 & 12.65 & 16.94 & 21.27 \\         
\hline                                                                  
TestTrajectoryError & 6.99 & 9.08 & 5.82 & 10.22 & 8.83 \\             
\hline                                                                  
TestVelocityError & 9.43 & 12.59 & 10.89 & 19.19 & 19.63 \\             
%
\hline                                                                  
DistanceToGoal & 3.76 & 1.30 & 1.92 & 5.66 & 6.17 \\                    
\hline                                                                  
DurationToGoal & - & 3.71 & 3.39 & 4.71 & 4.12 \\                      
\hline                                                                  
NumberReachedGoal & 0.00 & 7.00 & 2.00 & 7.00 & 7.00 \\                 
\hline                                                                  
GridDuration & - & 2.08 & 1.40 & 2.69 & 2.35 \\                        
\hline                                                                  
GridFractionReachedGoal & 0.00 & 1.0 & 0.38 & 1.00 & 1.00 \\           
\hline                                                                  
GridDistanceToGoal & 3.96 & 1.0 & 2.22 & 1.00 & 1.00 \\                
\hline          
GridDTWD & 25679.59 & 14652 & 121312.43 & 13680.30 & 12797.03 \\     
\hline 
TrainingTime & 0.02 & 5.34 & 6.11 & 9.56 & 7.28 \\   
\hline                                                                  
IntegrationSpeed & 0.21 & 0.03 & 0.04 & 0.17 & 0.06 \\
\hline
\end{tabular}%
\begin{tabular}{|c|c|}
     \hline
     CVF-CF & CVF-GS \\
     \hline
     12.00 & 8.87 \\
     \hline
     19.93 & 15.55 \\
     \hline
     8.13 & 11.17 \\
     \hline
     11.83 & 16.40 \\
     \hline
     0.27 & 3.00 \\
     \hline
     3.00 & 3.46 \\
     \hline
     6.00 & 7.00 \\
     \hline
     1.38 & 3.33 \\
     \hline
     1.00 & 1.00 \\
     \hline
     1.00 & 1.00 \\
     \hline
     15251.28 & 10944.76 \\
     \hline
     24.59 & 24.01 \\
     \hline
     0.05 & 0.03 \\
     \hline
     \end{tabular}
\caption{CShape}                                                        
\label{table:CShape}                                                \end{minipage}    \\
 \begin{minipage}{.5\textwidth}
 
\centering                                                            
\begin{tabular}{|c|c|c|cHH}                                        
\hline                                                                
 & DMP & SEDS & CLFDM & CVF-CF & CVF-GS \\                  
\hline                                                                
TrainingTrajectoryError & 3.81 & 7.84 & 6.41 & 7.60 & 7.54 \\         
\hline                                                                
TrainingVelocityError & 5.85 & 12.12 & 9.83 & 12.64 & 13.56 \\        
\hline                                                                
TestTrajectoryError & 5.56 & 10.44 & 11.00 & 10.84 & 7.86 \\          
\hline                                                                
TestVelocityError & 7.04 & 14.84 & 16.07 & 15.14 & 13.36 \\           
\hline                                                                
DistanceToGoal & 1.63 & 0.18 & 3.72 & 0.06 & 3.66 \\                  
\hline                                                                
DurationToGoal & - & 4.86 & 15.08 & 4.08 & 6.10 \\                    
\hline                                                                
NumberReachedGoal & 0.00 & 7.00 & 7.00 & 7.00 & 7.00 \\               
\hline                                                                
GridDuration & - & 2.78 & 12.68 & 1.92 & 4.06 \\                      
\hline                                                                
GridFractionReachedGoal & 0.00 & 1.00 & 1.00 & 1.00 & 1.00 \\         
\hline                                                                
GridDistanceToGoal & 1.71 & 1.00 & 1.00 & 1.00 & 1.00 \\              
\hline        
GridDTWD & 19085.48 & 11493 & 10928.89 & 12522.51 & 9109.60 \\     
\hline                                                                
TrainingTime & 0.02 & 14.43 & 3.01 & 12.13 & 16.77 \\      
\hline                                                                
IntegrationSpeed & 0.18 & 0.08 & 0.10 & 0.21 & 0.04 \\   
\hline
\end{tabular}%
\begin{tabular}{|c|c|}
     \hline
     CVF-CF & CVF-GS \\
     \hline
     10.07 & 6.83 \\
     \hline
     16.38 & 11.71 \\
     \hline
     5.82 & 6.29 \\
     \hline
     8.90 & 10.04 \\
     \hline
     0.18 & 0.11 \\
     \hline
     4.24 & 4.91 \\
     \hline
     7.00 & 7.00 \\
     \hline
     1.74 & 3.07 \\
     \hline
     0.94 & 1.00 \\
     \hline
     2.93 & 1.00 \\
     \hline
     17767.83 & 10637.81 \\
     \hline
     23.23 & 22.85 \\
     \hline
     0.10 & 0.03 \\
     \hline
     \end{tabular}
\caption{GShape}                                                      
\label{table:GShape}                                                  

\end{minipage}
 \begin{minipage}{.5\textwidth}
 
\centering                                                            
\begin{tabular}{|c|c|c|cHH}                                        
\hline                                                                
 & DMP & SEDS & CLFDM & CVF-CF & CVF-GS \\                  
\hline                                                                
TrainingTrajectoryError & 3.09 & 14.37 & 4.96 & 4.03 & 4.53 \\        
\hline                                                                
TrainingVelocityError & 6.05 & 22.88 & 9.60 & 9.53 & 10.12 \\         
\hline                                                                
TestTrajectoryError & 5.47 & 16.85 & 4.95 & 5.11 & 5.72 \\            
\hline                                                                
TestVelocityError & 8.49 & 25.76 & 9.81 & 11.40 & 12.67 \\            
\hline                                                                
DistanceToGoal & 2.24 & 0.00 & 6.07 & 1.68 & 3.04 \\                  
\hline                                                                
DurationToGoal & - & 2.05 & 3.49 & 3.58 & 3.90 \\                     
\hline                                                                
NumberReachedGoal & 0.00 & 7.00 & 5.00 & 7.00 & 7.00 \\               
\hline                                                                
GridDuration & 100.73 & 1.21 & 2.18 & 1.56 & 2.19 \\                  
\hline                                                                
GridFractionReachedGoal & 0.44 & 1.00 & 0.69 & 0.94 & 0.94 \\         
\hline                                                                
GridDistanceToGoal & 1.42 & 1.00 & 3.61 & 3.24 & 1.55 \\              
\hline       
GridDTWD & 29568.10 & 12866 & 15403.52 & 23123.40 & 12021.50 \\    
\hline                                                                
TrainingTime & 0.01 & 18.15 & 2.66 & 23.96 & 5.62 \\   
\hline                                                                
IntegrationSpeed & 0.19 & 0.11 & 0.09 & 0.12 & 0.02 \\  
\hline
\end{tabular}%
\begin{tabular}{|c|c|}
     \hline
     CVF-CF & CVF-GS \\
     \hline
     8.62 & 4.10 \\
     \hline
     14.78 & 9.31 \\
     \hline
     4.89 & 5.42 \\
     \hline
     11.28 & 11.51 \\
     \hline
     4.07 & 2.28 \\
     \hline
     3.49 & 3.79 \\
     \hline
     7.00 & 7.00 \\
     \hline
     1.56 & 3.62 \\
     \hline
     0.94 & 1.00 \\
     \hline
     3.21 & 1.00 \\
     \hline
     19610.23 & 10009.24 \\
     \hline
     22.59 & 25.20 \\
     \hline
     0.04 & 0.02 \\
     \hline
     \end{tabular}
\caption{JShape}                                                      
\label{table:JShape}                                                \end{minipage}  
\end{table*}                 
}
In Tables~\ref{table:Angle},~\ref{table:CShape},~\ref{table:GShape} and~\ref{table:JShape} we report comprehensive comparisons against 3 methods proposed in the literature: DMPs~\cite{ijspeert2013dynamical}, SEDS~\cite{khansari2011learning} and CLFDM~\cite{CLFDM}. We use publicly available implementations for these methods.
 
Our methods are abbreviated CVF-CF and CVF-GS, standing for contracting vector fields defined by curl-free and Gaussian separable random feature maps of Eqn.~\ref{eq:random_features_cf} and~\ref{eq:random_features_gs} respectively. In all experiments, we use the SCS solver described in section~\ref{sec:lmi_solver} for training, with $100$ or $200$ random features, bandwidth $\sigma$ set to $5$ or $10$, $\lambda$ tuned over $\{0.001, 0.01, 0.1\}$ and $\tau=0.0$. Contraction constraints were imposed on $250$ points. For each shape, we use $4$ demonstrations for training and $3$ demonstrations for testing. CVF and DMPs are trained on a single trajectory which is the average of the $4$ training demonstrations. 

Overall, it may be seen that CVF is highly competitive in comparison to other methods. On $3$ of the $4$ shapes, its mean trajectory error is among the lowest two. In terms of stability, across $64$ runs on the $4$ datasets starting from random points on a grid around the demonstrations, CVF methods return the best mean DTWD on all $4$ datasets.  We encountered only one case where an initial condition not converge to the goal. The time taken to reach the goal was also similar to the demonstration duration.  While CVF training time is slightly greater than other methods, SCS also returns good solutions with looser termination criteria making training time comparable to other methods. CVF inference speed is also in the same ballpark as other techniques. We noted that while DMP returns excellent trajectory and velocity errors, the norm of the velocity often shrinks prematurely considerably slowing down in the vicinity of the goal. We also noted that SEDS required data smoothing to return competitive results while being sensitive to initialization, since it uses a non-convex procedure to fit Gaussian mixture models. Unlike CLFDM and SEDS, CVF training involves convex optimization which if feasible has a unique globally optimal solution.

\subsection{Scalability wrt Dimensions and Random Features}
In the Figure below, we report how training time scales with increasing dimensionality $n$ and model capacity measured in terms of the number of random features $D$.  We embedded the "S" shape data into a random two-dimensional subspace of $\reals^n$ with $n=2, 4, 8, 16, 32$ for $D=300$ curl-free random features. In another experiment, we fixed $n=4$ and increased $D$ from $100$ to $500$. The SCS solver was run for $6000$ iterations. In all runs, the solver approached convergence achieving primal/dual residuals and duality gaps of the order of $10^{-6}$ to $10^{-4}$. In the regimes tested, the solver shows linear scaling with respect to $D$ and is superlinear with respect to $n$. These results confirm that our approach is practical for learning higher dimensional dynamical systems.
\begin{figure}[h]
    \centering
    \includegraphics[height=3.5cm, width=\linewidth]{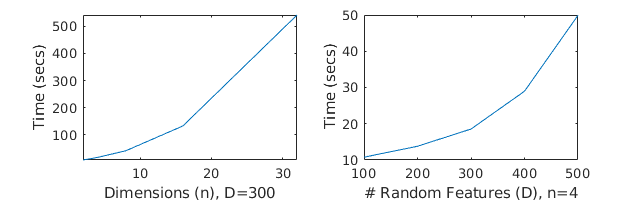}
    \label{fig:scalability}
\end{figure} \vspace{-0.7cm}

\subsection{Conclusion}
 Our approach is highly competitive with prior methods in learning-from-demonstration benchmarks, and brings together contraction analysis of nonlinear systems, vector-valued RKHS methods for statistical learning and random embeddings for fast convex optimization. Natural extensions of our work include: learning with more general contraction metrics, exploiting modularity properties~\cite{slotine2003modular}, exploring dynamic obstacle avoidance and coupling our approach with deep nets and perception modules for solving robotics tasks. 
\bibliographystyle{plainnat}
\bibliography{lyapunov}

\begin{thebibliography}{31}
\providecommand{\natexlab}[1]{#1}
\providecommand{\url}[1]{\texttt{#1}}
\expandafter\ifx\csname urlstyle\endcsname\relax
  \providecommand{\doi}[1]{doi: #1}\else
  \providecommand{\doi}{doi: \begingroup \urlstyle{rm}\Url}\fi

\bibitem[las()]{lasa}
\url{https://cs.stanford.edu/people/khansari/download.html}.

\bibitem[Ahmadi(2012)]{ahmadi2012algebraic}
Amir~Ali Ahmadi.
\newblock Algebraic relaxations and hardness results in polynomial optimization
  and lyapunov analysis.
\newblock \emph{arXiv preprint arXiv:1201.2892}, 2012.

\bibitem[Alvarez et~al.(2012)Alvarez, Rosasco, Lawrence,
  et~al.]{alvarez2012kernels}
Mauricio~A Alvarez, Lorenzo Rosasco, Neil~D Lawrence, et~al.
\newblock Kernels for vector-valued functions: A review.
\newblock \emph{Foundations and Trends{\textregistered} in Machine Learning},
  4\penalty0 (3):\penalty0 195--266, 2012.

\bibitem[Argall et~al.(2009)Argall, Chernova, Veloso, and
  Browning]{argall_survey_09}
Brenna~D Argall, S~Chernova, M~Veloso, and B~Browning.
\newblock A survey of robot learning from demonstration.
\newblock \emph{Robotics and autonomous systems}, 57\penalty0 (5):\penalty0
  469--483, 2009.

\bibitem[Aylward et~al.(2008)Aylward, Parrilo, and
  Slotine]{aylward2008stability}
Erin~M Aylward, Pablo~A Parrilo, and Jean-Jacques~E Slotine.
\newblock Stability and robustness analysis of nonlinear systems via
  contraction metrics and sos programming.
\newblock \emph{Automatica}, 44\penalty0 (8):\penalty0 2163--2170, 2008.

\bibitem[Berlinet and Thomas-Agnan(2011)]{berlinet2011reproducing}
Alain Berlinet and Christine Thomas-Agnan.
\newblock \emph{Reproducing kernel Hilbert spaces in probability and
  statistics}.
\newblock Springer Science \& Business Media, 2011.

\bibitem[Billard et~al.(2008)Billard, Calinon, Dillmann, and
  Schaal]{billard_08}
A.~Billard, S.~Calinon, R.~Dillmann, and S.~Schaal.
\newblock Survey: Robot programming by demonstration.
\newblock Handbook of Robotics, . chapter 59, 2008, 2008.

\bibitem[Brault et~al.(2016)Brault, Heinonen, and Buc]{brault2016random}
Romain Brault, Markus Heinonen, and Florence Buc.
\newblock Random fourier features for operator-valued kernels.
\newblock In \emph{Asian Conference on Machine Learning}, pages 110--125, 2016.

\bibitem[Giesl and Hafstein(2015)]{giesl2015review}
Peter Giesl and Sigurdur Hafstein.
\newblock Review on computational methods for lyapunov functions.
\newblock \emph{Discrete and Continuous Dynamical Systems-Series B},
  20\penalty0 (8):\penalty0 2291--2331, 2015.

\bibitem[Huang et~al.(2014)Huang, Avron, Sainath, Sindhwani, and
  Ramabhadran]{huang2014kernel}
Po-Sen Huang, Haim Avron, Tara~N Sainath, Vikas Sindhwani, and Bhuvana
  Ramabhadran.
\newblock Kernel methods match deep neural networks on timit.
\newblock In \emph{Acoustics, Speech and Signal Processing (ICASSP), 2014 IEEE
  International Conference on}, pages 205--209. IEEE, 2014.

\bibitem[Ijspeert et~al.(2013)Ijspeert, Nakanishi, Hoffmann, Pastor, and
  Schaal]{ijspeert2013dynamical}
Auke~Jan Ijspeert, Jun Nakanishi, Heiko Hoffmann, Peter Pastor, and Stefan
  Schaal.
\newblock Dynamical movement primitives: learning attractor models for motor
  behaviors.
\newblock \emph{Neural computation}, 25\penalty0 (2):\penalty0 328--373, 2013.

\bibitem[Jouffroy and Fossen(2010)]{jouffroy2010tutorial}
Jerome Jouffroy and Thor~I Fossen.
\newblock A tutorial on incremental stability analysis using contraction
  theory.
\newblock \emph{Modeling, Identification and control}, 31\penalty0
  (3):\penalty0 93, 2010.

\bibitem[Keogh and Ratanamahatana(2005)]{keogh2005exact}
Eamonn Keogh and Chotirat~Ann Ratanamahatana.
\newblock Exact indexing of dynamic time warping.
\newblock \emph{Knowledge and information systems}, 7\penalty0 (3):\penalty0
  358--386, 2005.

\bibitem[Khansari-Zadeh and Billard(2011)]{khansari2011learning}
S~Mohammad Khansari-Zadeh and Aude Billard.
\newblock Learning stable nonlinear dynamical systems with gaussian mixture
  models.
\newblock \emph{IEEE Transactions on Robotics}, 27\penalty0 (5):\penalty0
  943--957, 2011.

\bibitem[Khansari-Zadeh and Billard(2014)]{CLFDM}
S.~Mohammad Khansari-Zadeh and Aude Billard.
\newblock Learning control lyapunov function to ensure stability of dynamical
  system-based robot reaching motions.
\newblock \emph{Robotics and Autonomous Systems}, 6\penalty0 (62), 2014.

\bibitem[Khansari-Zadeh and Khatib(2017)]{khansari2017learning}
Seyed~Mohammad Khansari-Zadeh and Oussama Khatib.
\newblock Learning potential functions from human demonstrations with
  encapsulated dynamic and compliant behaviors.
\newblock \emph{Autonomous Robots}, 41\penalty0 (1):\penalty0 45--69, 2017.

\bibitem[Lemme et~al.(2015)Lemme, Meirovitch, Khansari-Zadeh, Flash, Billard,
  and Steil]{lemme2015open}
Andre Lemme, Yaron Meirovitch, Seyed~Mohammad Khansari-Zadeh, Tamar Flash, Aude
  Billard, and Jochen~J Steil.
\newblock Open-source benchmarking for learned reaching motion generation in
  robotics.
\newblock 2015.

\bibitem[Lohmiller and Slotine(1998)]{lohmiller1998contraction}
Winfried Lohmiller and Jean-Jacques~E Slotine.
\newblock On contraction analysis for non-linear systems.
\newblock \emph{Automatica}, 34\penalty0 (6):\penalty0 683--696, 1998.

\bibitem[Macedo and Castro(2014)]{curlfree}
Ives Macedo and Rener Castro.
\newblock Learning divergence-free and curl-free vector fields with
  matrix-valued kernels.
\newblock \emph{Robotics and Autonomous Systems}, 6\penalty0 (62), 2014.

\bibitem[Micchelli and Pontil(2005)]{micchelli2005learning}
Charles~A Micchelli and Massimiliano Pontil.
\newblock On learning vector-valued functions.
\newblock \emph{Neural computation}, 17\penalty0 (1):\penalty0 177--204, 2005.

\bibitem[Micheli and Glaunes(2013)]{micheli2013matrix}
Mario Micheli and Joan~Alexis Glaunes.
\newblock Matrix-valued kernels for shape deformation analysis.
\newblock \emph{arXiv preprint arXiv:1308.5739}, 2013.

\bibitem[Minh(2016)]{minh2016operator}
Ha~Quang Minh.
\newblock Operator-valued bochner theorem, fourier feature maps for
  operator-valued kernels, and vector-valued learning.
\newblock \emph{arXiv preprint arXiv:1608.05639}, 2016.

\bibitem[O’Donoghue et~al.(2016)O’Donoghue, Chu, Parikh, and
  Boyd]{o2016conic}
Brendan O’Donoghue, Eric Chu, Neal Parikh, and Stephen Boyd.
\newblock Conic optimization via operator splitting and homogeneous self-dual
  embedding.
\newblock \emph{Journal of Optimization Theory and Applications}, 169\penalty0
  (3):\penalty0 1042--1068, 2016.

\bibitem[Rahimi and Recht(2008)]{rahimi2008random}
Ali Rahimi and Benjamin Recht.
\newblock Random features for large-scale kernel machines.
\newblock In \emph{Advances in neural information processing systems}, pages
  1177--1184, 2008.

\bibitem[Ravichandar et~al.(2017)Ravichandar, Salehi, and Dani]{harish}
Harish Ravichandar, Iman Salehi, and Ashwin Dani.
\newblock Learning partially contracting dynamical systems from demonstrations.
\newblock In \emph{Conference on Robot Learning (CoRL)}, 2017.

\bibitem[Scholkopf and Smola(2001)]{scholkopf2001learning}
Bernhard Scholkopf and Alexander~J Smola.
\newblock \emph{Learning with kernels: support vector machines, regularization,
  optimization, and beyond}.
\newblock 2001.

\bibitem[Schwartz(1964)]{schwartz}
Laurent Schwartz.
\newblock Sous-espaces hilbertiens d{'}espaces vectoriels topologiques et
  noyaux associ\'es (noyaux reproduisants).
\newblock \emph{J. Analyse Math.}, 1964.

\bibitem[Sindhwani et~al.(2012)Sindhwani, Quang, and
  Lozano]{sindhwani2012scalable}
Vikas Sindhwani, Minh~Ha Quang, and Aur{\'e}lie~C Lozano.
\newblock Scalable matrix-valued kernel learning for high-dimensional nonlinear
  multivariate regression and granger causality.
\newblock \emph{arXiv preprint arXiv:1210.4792}, 2012.

\bibitem[Slotine(2003)]{slotine2003modular}
Jean-Jacques~E Slotine.
\newblock Modular stability tools for distributed computation and control.
\newblock \emph{International Journal of Adaptive Control and Signal
  Processing}, 17\penalty0 (6):\penalty0 397--416, 2003.

\bibitem[Slotine et~al.(1991)Slotine, Li, et~al.]{slotine1991applied}
Jean-Jacques~E Slotine, Weiping Li, et~al.
\newblock \emph{Applied nonlinear control}, volume 199.
\newblock Prentice hall Englewood Cliffs, NJ, 1991.

\bibitem[Zhou(2008)]{zhou2008derivative}
Ding-Xuan Zhou.
\newblock Derivative reproducing properties for kernel methods in learning
  theory.
\newblock \emph{Journal of computational and Applied Mathematics}, 220\penalty0
  (1-2):\penalty0 456--463, 2008.

\end{thebibliography}


\end{document}